\documentclass[nohyperref]{article}

\usepackage{microtype}

\usepackage{graphicx}
\usepackage{multirow}
\usepackage{subfigure}
\usepackage{booktabs} 

\usepackage{hyperref}

\usepackage[accepted]{icml2022}

\usepackage{amsmath}
\usepackage{amssymb}
\usepackage{pmat}
\usepackage{mathtools}
\usepackage{amsthm}

\usepackage[capitalize,noabbrev]{cleveref}

\theoremstyle{plain}
\newtheorem{theorem}{Theorem}[section]
\newtheorem{proposition}[theorem]{Proposition}

\theoremstyle{definition}
\newtheorem{definition}[theorem]{Definition}

\theoremstyle{remark}

\usepackage[textsize=tiny]{todonotes}

\icmltitlerunning{Omni-Granular Ego-Semantic Propagation for Self-Supervised Graph Representation Learning}

\begin{document}

\twocolumn[
\icmltitle{Omni-Granular Ego-Semantic Propagation for Self-Supervised \\ Graph Representation Learning}




\begin{icmlauthorlist}
\icmlauthor{Ling Yang}{yyy,sch}
\icmlauthor{Shenda Hong}{yyy,sch}
\end{icmlauthorlist}

\icmlaffiliation{sch}{Institute of Medical Technology, Health Science Center of Peking University, Beijing, China}
\icmlaffiliation{yyy}{National Institute of Health Data Science, Peking University, Beijing, China}

\icmlcorrespondingauthor{Ling Yang}{yangling0818@163.com}
\icmlcorrespondingauthor{Shenda Hong}{hongshenda@pku.edu.cn}

\icmlkeywords{Machine Learning, ICML}

\vskip 0.3in
]



\printAffiliationsAndNotice{}  

\begin{abstract}
Unsupervised/self-supervised graph representation learning is critical for downstream node- and graph-level classification tasks. 
Global structure of graphs helps discriminating representations and existing methods mainly utilize the global structure by imposing additional supervisions. However, their global semantics are usually invariant for all nodes/graphs and they fail to explicitly embed the global semantics to enrich the representations. In this paper, we propose \textbf{O}mni-Granular \textbf{E}go-Semantic \textbf{P}ropagation for Self-Supervised \textbf{G}raph Representation Learning (\textit{OEPG}). Specifically, we introduce instance-adaptive global-aware \textit{ego-semantic} descriptors, leveraging the first- and second-order feature differences between each node/graph and hierarchical global clusters of the entire graph dataset. The descriptors can be explicitly integrated into local graph convolution as new neighbor nodes. Besides, we design an \textit{omni-granular} normalization on the whole scales and hierarchies of the ego-semantic to assign attentional weight to each descriptor from an omni-granular perspective.  
Specialized pretext tasks and cross-iteration momentum update are further developed for local-global mutual adaptation. In downstream tasks, OEPG consistently achieves the best performance with \textbf{a 2\%$\sim$6\% accuracy gain} on multiple datasets cross scales and domains. Notably, OEPG also generalizes to quantity- and topology-imbalance scenarios.
\end{abstract}

\section{Introduction}
In the past few years, the paradigm of graph learning has been shifted from structural pattern discovery \cite{leskovec2005graphs,borgatti2000models,milo2004superfamilies,newman2006modularity,watts1998collective} to graph representation learning \cite{hamilton2017inductive,velivckovic2017graph}. In many scientific domains and industrial scenarios, labeled graph-structured data is limited and hard to obtain. Thus learning the representations in an unsupervised or self-supervised manner becomes increasingly important. Traditional unsupervised graph representation learning approaches, such as
DeepWalk \cite{perozzi2014deepwalk}, node2vec \cite{grover2016node2vec} and LINE \cite{tang2015line}, utilize a framework originated in the skip-gram model \cite{mikolov2013distributed} which aim to learn an encoding function for transforming nodes to low-dimensional embeddings that preserve vital attributive and structural features. Recent works conduct graph representation learning in different ways including learning transferable prior knowledge with meta-learning \cite{lu2021learning,thakoor2021bootstrapped} and predicting the informative substructures of graph \cite{hu2019strategies,jiao2020sub,zhang2020motif,zhang2021motif}.

Inspired by powerful contrastive learning paradigm \cite{oord2018representation,chen2020simple,he2020momentum}, graph contrastive learning (GCL) methods \cite{zhu2020deep} achieve a significant success. Recent works mainly focus on designing topology \cite{hassani2020contrastive,qiu2020gcc} or node/edge augmentations \cite{you2020graph,zhu2021graph} and defining the contrastive pairs at different levels (e.g., node-level and graph-level) of the graph \cite{peng2020graph,sun2019infograph}. 
They perform local-global or global-global contrasting \cite{zhu2021empirical} to discriminate the representations, proving the effectiveness of global semantics. However, their use of the whole adjacency matrix is impractical when scaling to large graphs.
\begin{figure*}[ht]
\vskip 0.1in
\begin{center}\centerline{\includegraphics[width=12cm]{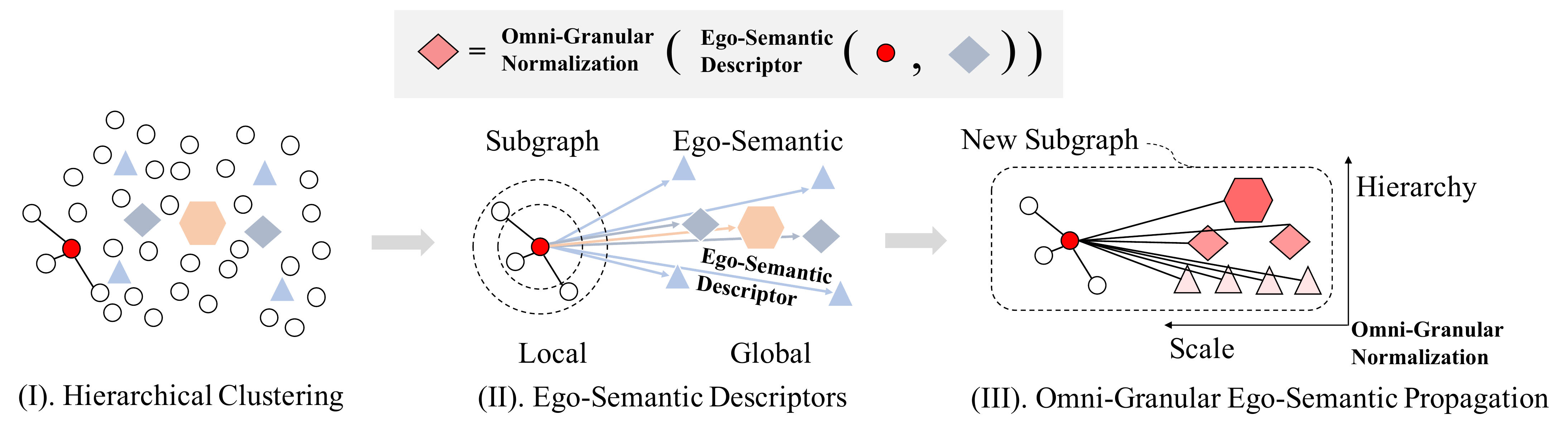}}
\caption{Illustration of OEPG. (I) Obtain hierarchical global clusters. (II) Explicitly define instance-adaptive global-aware ego-semantic descriptors by measuring first- and second-order feature differences between target node (denoted in red) and hierarchical clusters. (III) Perform omni-granular normalization on ego-semantic descriptors and use them to form new subgraph for later feature propagation.}
\label{pic-oepg}
\end{center}
\vskip -0.1in
\end{figure*}
Some recent global-level self-supervised methods \cite{you2020does} focus on utilizing the global semantics to strengthen the graph representation learning, using the global structure information to guide the local pre-training \cite{jin2020self} or impose the additional supervisions to conduct informative learning \cite{xu2021self}. Nevertheless, their utilization of global semantics are implicit since they can not explicitly embed the global structure information into representations. Besides, their global semantics are invariant for all nodes/graphs, which fail to adequately describe the global diversity for better discriminativity and expressiveness.  

In this paper, we propose Omni-Granular Ego-Semantic Propagation for self-supervised graph representation learning (OEPG) to address above problems.
As in Figure.\ref{pic-oepg}-(I), we use k-means to obtain global clusters on hierarchical levels after node- or graph-level pre-training of the graph dataset.
We characterize the instance-adaptive global-aware feature for each node/graph by the \textit{ego-semantic} descriptor, which leverages the first- and second-order feature differences between itself and all global clusters in a 1-vs-N manner as in Figure.\ref{pic-oepg}-(II). 
We do not straightly use global clusters to represent the global semantic since it should be instance-adaptive from the ego perspective of different nodes/graphs.
We explicitly incorporate the ego-semantic descriptor into local graph convolutional process to model the local-global feature propagation. 
Contrary to implicit methods such as predicting global structure related pairwise distances \cite{jin2020self} or imposing additional global structural supervision \cite{xu2021self}, our ego-semantic descriptors can be explicitly utilized to make feature propagation between local context and global semantic. 

We devise an \textit{omni-granular} normalization over the hierarchies and scales of the ego-semantic to make cross-granularity attention between the coarse-grained and fine-grained descriptors. It assigns attentional weight to each descriptor according to the contribution from an omni-granular perspective. After the normalization, the weighted ego-semantic descriptors and original graph (or subgraph centered around node) are used to form a new one as in Figure.\ref{pic-oepg}-(III). To adequately utilize the local and global information of the new graph/subgraph,
we further propose specialized pretext tasks involving local neighbors and global ego-semantic descriptors by performing mutual contrasting and prediction.
Considering the dynamics of feature space in training process, we design a cross-iteration momentum update mechanism for omni-granular ego-semantic descriptors in a moving-average way to achieve local-global mutual adaptation.
It is noteworthy that the ego-semantic and omni-granular normalization of our OEPG are also applicable for downstream tasks since they are explicit feature or operations, and thus narrow the gap between pre-training and applications.
Additionally, OEPG is robust to quantity- and topology-imbalance \cite{chen2021topology} scenarios without additional operations (e.g., re-weighting or over-sampling).

Our main contributions are summarized as the following:
\begin{itemize}
    \item We firstly explicitly characterize the instance-adaptive global-aware feature by \textit{ego-semantic} descriptors, utilizing combined first- and second-order feature differences to discriminate and enrich the representations.
    \item We propose an \textit{omni-granular} normalization over all the hierarchies and scales of ego-semantic for assigning attentional weight to each descriptor according to the contribution from an omni-granular perspective.
    \item Specialized tasks and a cross-iteration omni-granular momentum update are both proposed for achieving local-global mutual adaptation in our framework.
    \item Our OEPG substantially outperforms previous methods with a 2\%$\sim$6\% accuracy gain in multiple downstream tasks on datasets cross scales and domains, and generalizes to quantity- and topology-imbalance scenarios.
\end{itemize}

\section{Related Work}
\paragraph{Self-Supervised Graph Representation Learning.} Existing self-supervised graph representation learning methods can be mainly divided into two categories: contrastive \cite{you2020graph} and predictive \cite{hu2019strategies,wu2021self}. The contrastive methods mainly design different augmented views and build loss function \cite{velivckovic2018deep,sun2019infograph,peng2020graph} based on the pre-defined contrastive pairs to guide the graph representation learning \cite{suresh2021adversarial,xu2021infogcl,zhang2021canonical}. 
The predictive methods focus on the information embedded in the graph, generally based on
pretext tasks such as reconstruction \cite{kipf2016variational} or other predictive tasks, which exploit the attributes and
structures of graph data as additional supervisions \cite{hu2020gpt,rong2020self,xu2021self,thakoor2021bootstrapped}. 
\paragraph{Global-Level Self-Supervised Objectives.} Recent works begin to focus on global-level self-supervised objectives, utilizing the global structure of the whole graph dataset. Some methods use the whole adjacency matrix to do the local-global contrastive learning \cite{zhu2021empirical}, which is impractical in large graphs.  \citet{you2020does} applies graph topology partition and graph completion to improve the self-supervision. \citet{jin2020self} takes a bird’s-eye view of the position of the node in the graph and proposes to pre-train by the guidance of global structure information. \citet{xu2021self} learns the global structure by an online expectation-maximization (EM)
algorithm \cite{dempster1977maximum}. Nevertheless, their approaches to capture the global semantic are all implicit with additional supervisions, and thus they are inapplicable in the inference stage and lead to a gap between pre-training and applications.

\section{OEPG: Omni-Granular Ego-Semantic Propagation Graph}
\label{sec-oepg}
In this section, we elucidate our Omni-Granular Ego-Semantic Propagation Graph (OEPG) in detail.
OEPG is designed to explicitly characterize the instance-adaptive global-aware feature for node/graph representations.
Specifically, we derive first- and second-order \textit{ego-semantic} descriptors for each subgraph in a locally-smoothed feature space, and perform \textit{omni-granular} normalization on the descriptors to make attentional cross-granularity interactions.  
Then we aggregate the original subgraph and weighted ego-semantic (combine first and second orders) to form a new subgraph.
Finally, specialized pretext tasks and a cross-iteration momentum update mechanism are devised for local-global feature propagation and mutual adaptation.

\paragraph{A Unified Perspective.} In our framework, we treat the node- and graph-level graph representation learning from an unified perspective. Specifically, considering a set of unlabeled graphs or subgraphs $\mathbf{G}=\{\mathcal{G}_1,\mathcal{G}_2,\cdots,\mathcal{G}_N\}$, where a subgraph $\mathcal{G}=(\mathcal{V},\mathcal{E})$ is centered around node $\mathbf{V}$. 
\textbf{We only illustrate the node-level pre-training for simplicity.}
We aim at learning low-dimensional vectors $\mathbf{H}=\{h_{\mathcal{G}_1},h_{\mathcal{G}_2},\cdots,h_{\mathcal{G}_N}\}$, where $h_{\mathcal{G}_N}\in \mathbb{R}^d$. The embedding process of a $l$-th GNN model is formulated as:  
\begin{align}
\label{eq-pre}
\begin{split}
    \mathbf{V}_l=\mathcal{F}_1(\mathbf{V}_{l-1},&\mathcal{F}_2(\{(\mathbf{V}^{u}_{l-1}, \mathbf{V}_{l-1},\mathcal{E})):u\in \mathcal{N}_{\mathbf{V}}\}),\\
    &h_{\mathcal{G}}=\mathcal{F}_R({\mathbf{V}|\mathbf{V}\in \mathcal{V}}),
\end{split}
\end{align}
where $\mathcal{N}_{\mathbf{V}}$ is the neighbor set of $\mathbf{V}$, $\mathcal{F}_1$ and $\mathcal{F}_2$ denote the message passing and update function \cite{gilmer2017neural} respectively. $\mathcal{F}_R$ is a permutation-invariant readout function \cite{ying2018hierarchical,zhang2018end}.
\subsection{Ego-Semantic Descriptor}
\label{sub3.1}
\paragraph{Motivation.} Graph Neural Networks (GNNs) derive proximity-preserved feature vectors in a neighborhood aggregation way \cite{gilmer2017neural}.
Such a locally-smoothed latent space benefits the information propagation \cite{kipf2016semi,garcia2017learning,wang2020unifying} between the nodes. However, the local proximity based feature representations fail to capture the global pattern or structure of the whole graph, which is crucial for learning discriminative representations, especially in a self-supervised setting. An alternative way is to utilize the whole adjacency matrix \cite{hjelm2018learning,hassani2020contrastive}, but it is impractical for when applied in large graph. \citet{jin2020self} and \citet{xu2021self} firstly use global structure related supervisions for representation learning. We note that their utilization of global information are implicit, and the global structure information can not be explicitly embedded in learned representations. Besides, their global semantics are invariant for all subgraph instances, limiting the discriminativity and expressiveness.
We argue that an effective expression of global semantic should be instance-adaptive and dynamic, which means it changes accordingly with the change of target node/graph and thus helps discriminating  the representations from both local and global perspectives. Hence, we explicitly characterize an instance-adaptive global-aware feature, which integrates the local-global semantic differences.

\paragraph{The First-Order Ego-Semantic Descriptor.} Considering that existing global-level methods only 
use instance-invariant global information and also lead to over-smoothing problem, we propose the first-order ego-semantic descriptor to embed the \textbf{instance-adaptive 1-vs-N global diversity} with differential feature and alleviate the over-smoothing problem, which firstly considers the differences between local subgraph feature and global semantic clusters. 
In specific, we first use typical graph contrastive learning to acquire a locally-smoothed feature space and use K-means to make hierarchical clustering on the whole graph dataset for obtaining non-parametric global clusters $\{\{\mathbf{C}_{s,h}\}_{s=1}^{S_h}\}_{h=1}^{H}$, where $H$ is the total hierarchies and $S_h$ denotes the number of scales in $h$-th hierarchy ($S_1>S_2>\cdots>S_H$). After that, we not only want to associate the local context with the global semantic but also expect that each target node can explicitly embeds the self-adaptive global-aware semantic. To achieve this goal, we define the first-order ego-semantic descriptor.
\begin{definition}
\label{def:1}
For $s$-th cluster in $h$-th hierarchy, where $h \in [1,H]$, $s \in [1,S_h]$, then the first-order ego-semantic descriptor $\mathbf{D}^{1st}_{s,h}\in \mathbb{R}_d$ for target $\mathbf{V}^{target}$ is defined as: 
\begin{align}
\label{ego-semantic1}
\begin{split}
    \mathbf{D}^{1st}_{s,h}=(\mathbf{V}^{target}-\mathbf{C}_{s,h})/||\mathbf{V}^{target}-\mathbf{C}_{s,h}||_2,
\end{split}
\end{align}
where we derive $\mathbf{V}^{target}/||\mathbf{V}^{target}-\mathbf{C}_{s,h}||_2=f^t\in \mathbb{R}^d$, $\mathbf{C}_{s,h}/||\mathbf{V}^{target}-\mathbf{C}_{s,h}||_2=f^c_{s,h}\in \mathbb{R}^d$, and we specify $\mathbf{D}^{1st}_{s,h}$ as ($[d]$ denotes $d$-th dimension):
\begin{align}
    [
    f^{t}[1]-f^{c}_{s,h}[1],\cdots,f^{t}[d]-f^{c}_{s,h}[d]
    ].
\end{align}
\end{definition}
$\{\{\mathbf{D}^{1st}_{s,h}\}_{s=1}^{S_h}\}_{h=1}^{H}$ describe diverse global semantics from the ego perspective, calculated by the channel-wise feature difference between the target node and full global clusters. Thus for target nodes of different classes, their first-order ego-semantic descriptors can be also different, discriminating the final representations. Note that we use $l_2$ normalization along the feature channels to eliminate the influence caused by large numerical differences of different $\mathbf{D}^{1st}_{s,h}$. 

\paragraph{The Second-Order Ego-Semantic Descriptor.} Regarding enhancing the discriminativity and alleviating the over-smoothing, the first-order ego-semantic descriptor only leverages the instance-wise global diversity. Further, we try to model the \textbf{instance-adaptive 1-vs-N global diversity correlations} from an omni-granular perspective, which aims to leverage more structural information compared with first-order ones. Thus we propose the second-order ego-semantic descriptor.
\begin{definition}
\label{def:2}
For $s$-th cluster in $h$-th hierarchy, where $h \in [1,H]$, $s \in [1,S_h]$, the second-order ego-semantic descriptor $\mathbf{D}^{2nd}_{s,h}\in \mathbb{R}^{(S_1+S_2+\cdots+S_H)}$ for $\mathbf{V}^{target}$ is:
\begin{align}
\label{ego-semantic2}
\begin{split}
    \mathbf{X} = [\mathbf{D}^{1st}_{1,1}\bullet\mathbf{D}^{1st}_{s,h},&\cdots,\mathbf{D}^{1st}_{S_H,H}\bullet\mathbf{D}^{1st}_{s,h}],\\
    \mathbf{D}^{2nd}_{s,h} =&\  \mathbf{X}/||\mathbf{X}||_2,
\end{split}
\end{align}
where $\bullet$ is the inner product. $\mathbf{X}$ is further specified as:
\begin{align}
\label{eq-1vn}
\begin{split}
    \big[\sum^{d}_{i=1}(f^{t}[i]-f^{c}_{1,1}[i])(f^{t}[i]-f^{c}_{s,h}[i]),\cdots,\\\sum^{d}_{i=1}(f^{t}[i]-f^{c}_{S_H,H}[i])(f^{t}[i]-f^{c}_{s,h}[i])\big]
\end{split}
\end{align}
\end{definition}
$\{\{\mathbf{D}^{2nd}_{s,h}\}_{s=1}^{S_h}\}_{h=1}^{H}$ provide target node with a set of global similarity-distribution based features, which model the dense 1-vs-N correlations between each first-order descriptor $\mathbf{D}^{1st}_{s,h}$ and $\{\{\mathbf{D}^{1st}_{s,h}\}_{s=1}^{S_h}\}_{h=1}^{H}$ from the ego perspective. Such second-order descriptors integrate more contextual information to enrich the representations. 
Both first- and second-order descriptors discriminate the node- and graph-level representations, further proved in Appendix.\ref{app-prove}. We combine the both to generate the final ego-semantic descriptors set $\{\{\mathbf{D}_{s,h}\}_{s=1}^{S_h}\}_{h=1}^{H}$ in Subsection.\ref{sec-omni}.

To the best of our knowledge, we are the first to explicitly characterize the instance-adaptive global-aware semantic features for GNNs. Notably, they can be easily utilized in existing GNN architectures. 
A typical local graph convolution in Eq.\ref{eq-pre} can be rewritten as ($\mathcal{E}$ is omitted for simplicity):
\begin{align}
\label{ego-semantic}
    \mathbf{V}^{target}_{l}=\mathcal{F}(\underbrace{\mathbf{V}^{target}_{l-1}, \{\mathbf{V}^{local}_{i,l-1}\}_{i=1}^{N_n}}_{local\ context}),
\end{align}
where $\{\mathbf{V}^{local}_{i,l-1}\}_{i=1}^{N_n}$ denotes the set of total $N_n$ neighbor nodes at the $l$-th layer in subgraph, $\mathcal{F}$ denotes the aggregation and update functions. 
We note that previous GNN methods can only use the local context to aggregate feature and propagate the information for target node.
In contrast, our final ego-semantic descriptors can be explicitly integrated into the local graph convolutional process as follow:
\begin{align}
\label{ego-local-global}
    \mathbf{V}^{target}_{l}=\mathcal{F}(\underbrace{\mathbf{V}^{target}_{l-1}, \{\mathbf{V}^{local}_{i,l-1}\}_{i=1}^{N_n}}_{local\ context},\underbrace{\{\{\mathbf{D}_{s,h}\}_{s=1}^{S_h}\}_{h=1}^{H}}_{global\ ego-semantic}).
\end{align}
Thus the target node explicitly embeds both local context and node-adaptive global semantics, improving the discriminativity and expressiveness from different perspectives.
\subsection{Omni-Granular Normalization}
\label{sec-omni}
The ego-semantic descriptors of each hierarchy
jointly compose the global semantic space and each descriptor represents a semantic patch. The semantic patch changes from the fine granularity to the coarse one with the hierarchy changing from the bottom to the top. 
The ego-semantics of multiple granularities are critical for many downstream tasks such as community detection and molecular property prediction.
Besides, different target node may have the different priorities of the ego-semantic granularities. Thus we perform an \textit{omni-granular} normalization on the whole scales and hierarchies of the ego-semantic to conduct cross-granularity attention between fine-grained and coarse-grained descriptors, and weight them according to their contributions. We normalize first- and second-order descriptors separately since they describe different characteristics of the ego-semantic.
For the first-order ego-semantic descriptors, we have:
\begin{align}
\label{omni-1}
\begin{split}
    a^{1st}_{s,h}
    =&\frac{e^{-\alpha||\mathbf{D}^{1st}_{s,h}||^2}}{\sum\limits^{H}_{m=1}\sum\limits^{S_m}_{k=1}e^{-\alpha||\mathbf{D}^{1st}_{k,m}||^2}},
    \\
\end{split}
\end{align}
where $\mathbf{D}^{1st}_{s,h}$ denotes the initial first-order ego-semantic descriptor. For each target node, we flatten the hierarchies and scales of the ego-semantic descriptors and calculate the attention weight $a_{s,h}$ (ranges between 0 and 1) with the omni-granular normalization technique. $\alpha$ is
a trainable parameter (positive constant) that controls the decay of the response with the magnitude of the feature difference. 
We weight the each initial first-order ego-semantic descriptor:
\begin{align}
\label{omni-4}
\begin{split}
    \mathbf{\Bar{D}}^{1st}_{s,h}=&a^{1st}_{s,h}\cdot\mathbf{D}^{1st}_{s,h}=\frac{e^{-\alpha||\mathbf{D}^{1st}_{s,h}||^2}}{\sum\limits^{H}_{m=1}\sum\limits^{S_m}_{k=1}e^{-\alpha||\mathbf{D}^{1st}_{k,m}||^2}}\cdot\mathbf{D}^{1st}_{s,h},
\end{split}
\end{align}
and we weight the each second-order ego-semantic descriptor in a similar way ($\beta$ is another trainable parameter as $\alpha$):
\begin{equation}
\begin{gathered}
\label{omni-5}
    \mathbf{\Bar{D}}^{2nd}_{s,h}
    =\frac{e^{-\beta||\mathbf{D}^{2nd}_{s,h}||^2}}{\sum\limits^{H}_{m=1}\sum\limits^{S_m}_{k=1}e^{-\beta||\mathbf{D}^{2nd}_{k,m}||^2}}\cdot\mathbf{D}^{2nd}_{s,h}.
\end{gathered}
\end{equation}
\begin{proposition}
\label{proposition}
The value of the trainable parameters $\alpha \in (0,\infty)$ and $\beta\in (0,\infty)$ could reflect the global semantical diversity from the ego perspective.
\end{proposition}
\begin{proof} 
If $\alpha$ $\rightarrow \infty$, the attention weight for the ego-semantic descriptor $\mathbf{D}^{1st}_{s,h}$ will be $a^{1st}_{s,h} \rightarrow 1$ and it will be $0$ for other ego-semantic descriptors, which means the target node is only related to the closest cluster $\mathbf{C}_{s,h}$ because of the global diversity of the data distribution. Conversely, If $\alpha$ $\rightarrow 0$, $a^{1st}_{s,h} \rightarrow \frac{1}{S_1+S_2+\cdots+S_H}$, which means all clusters have the similar contribution to the target node because of the global proximity of the data distribution. It is the same for $\beta$.
\end{proof}
We concatenate each weighted first- and second-order ego-semantic descriptor along the channel for better description:
\begin{align}
\label{omni-6}
    \mathbf{D}_{s,h} = \text{LeakyReLU}(\mathbf{W}\cdot\text{Concat}(\mathbf{\Bar{D}}^{1st}_{s,h},\mathbf{\Bar{D}}^{2nd}_{s,h})),
\end{align}
where a linear $\mathbf{W}:\mathbb{R}^{(d+S_1+S_2+\cdots+S_H)}\rightarrow\mathbb{R}^d$ and non-linear function follow the concatenation to fuse descriptors. Through above processes, each target captures the sufficient characteristics about diverse global information, which are beneficial for downstream classification tasks.

Finally, we group the local neighbors and projected ego-semantic descriptors $\mathbf{D}_{s,h}$ to construct a new subgraph for the target node. The adjacency matrix $A_{ori}$ of original subgraph is extended to the new adjacency matrix $A$:
\begin{align}
\label{eq-adj}
\begin{small}
\begin{pmat}[{|}]
A[local,local]=A_{ori} & A[target,global]=1\cr &A[neighbors,global]=0\cr\-
A[global,target]=1&\cr A[global,neighbors]=0&A[global,global]=0\cr
\end{pmat}
\end{small}
\end{align}
where $local$ are the indexes of $target$ and $neighbor$ nodes, and $global$ are the indexes of ego-semantic descriptors. The edges between the target and all ego-semantic descriptors are $1$ and the edges are set to $0$ between local neighbors and the descriptors.
For the graph-level representation learning, we treat the readout feature as the $\mathbf{V}^{target}$ and the $A$ is:
\begin{align}
\label{eq-adj}
\begin{small}
\begin{pmat}[{|}]
A[local,local]=A_{ori} & A[local,global]=1\cr \-
A[global,local]=1&A[global,global]=0\cr
\end{pmat}
\end{small}
\end{align}
The new subgraph contains both local and global information for later self-supervised graph representation learning.
The graph convolutional process would propagate the global-aware omni-granular ego-semantic information to the local context for discriminating the feature representations.   


\subsection{Model Optimization and Downstream Application}
\label{sub3.3}
\paragraph{Specialized Local-Global Pretext Tasks.} For better simultaneously leveraging both local and global information for embedding new subgraph, specialized pretext tasks are devised for pre-training. They can be grouped into two categories, predictive and contrastive pretext tasks. Node-based and edge-based contrastive modes have been fully explored to maximize the agreement of two graph augmentations in graph contrastive learning \cite{zhu2021empirical}. 
We extend them to make it controllable for our OEPG. In specific, we simultaneously drop random portion of the local neighbors (edges) and diverse global-aware ego-semantic descriptors (edges) to produce augmented subgraphs for local-global contrastive learning. Given the local-global masking $T_i,T_j$, the loss is: 
\begin{equation}
\label{eq-contra}
\begin{split}
    \mathcal{L}_{contrastive} =& -\mathbb{E}_{\mathcal{G}\sim p(\mathcal{G})}[sim(\mathcal{F}(T_i(\mathcal{G})),\mathcal{F}(T_j(\mathcal{G})))] + \\ &\mathbb{E}_{\mathcal{G}_{-}\sim p(\mathcal{G}_{-})}[sim(\mathcal{F}(T_i(\mathcal{G})),\mathcal{F}(T_j(\mathcal{G_{-}})))],
\end{split}
\end{equation}
where $sim(\cdot)$ is the similarity function and $\mathcal{G_{-}}$ is negative subgraph. It enables the model robust to both local and global semantic noises and effectively combines the local proximity and global diversity for discriminating the feature.

Context prediction task has been well applied in self-supervised graph representation learning \cite{hu2019strategies}. It enables the model to capture the local context by predicting surrounding graph structures. Inspired by this, we propose a cross-reconstruction task where we mask a substructure containing both local nodes and global ego-semantic descriptors in the subgraph, and use the rest subgraph $\mathcal{G}_{1}$ to reconstruct the masked substructure $\mathcal{G}_{2}$. The loss is: 
\begin{align}
\label{eq-predic}
\begin{split}
    \mathcal{L}_{predictive} =& -\mathbb{E}_{\mathcal{G}\sim p(\mathcal{G})}\ log(\sigma(\mathcal{F}(\mathcal{G}_{1})\cdot\mathcal{F}_{aux}(\mathcal{G}_{2}))),
\end{split}
\end{align}
where $\mathcal{G}_{1}\cup \mathcal{G}_{2}=\mathcal{G}$, $\mathcal{F}_{aux}$ is an auxiliary encoder to produce substructure embedding and $\sigma(\cdot)$ is the sigmoid function. It builds a strong correlation between local and global semantics by predicting the co-concurrences of them.

\paragraph{Cross-Iteration Omni-Granular Momentum Update.} In the pre-training procedure guided by specialized contrastive and predictive pretext tasks, 
Considering the shifts of global feature space in pre-training process, we introduce a cross-iteration momentum update mechanism of multi-granular clusters $\{\{\mathbf{C}_{s,h}\}_{s=1}^{S_h}\}_{h=1}^{H}$. 
Specifically, for each node, we assign it to a closest cluster of each hierarchy. Given the hierarchy index $h\in [1,H]$, then the scale index $s$ of each hierarchy is calculated by:
\begin{align}
\label{argmax}
\begin{split}
  s = \mathop{\text{argmax}}\limits_{k}(sim(\mathbf{V},\{\mathbf{C}_{k,h}\}_{k=1}^{S_h})).
\end{split}
\end{align}
Instead of updating node queue after each training iteration, we construct a node queue $\{\mathbf{Q}_{s,h,i}\}_{i=1}^{i\in [1,\mathcal{B}]}$ with a budget $\mathcal{B}$ for each cluster. If the number of nodes in a queue equals to the $\mathcal{B}$, the corresponding cluster will be update as follow: 
\begin{align}
\label{momentum}
\begin{split}
    \mathbf{C}_{s,h}\longleftarrow m\mathbf{C}_{s,h}+\frac{(1-m)}{\mathcal{B}}\sum\limits_{i=1}^{\mathcal{B}}\mathbf{Q}_{s,h,i},
\end{split}
\end{align}
where momentum coefficient m is set 0.999 in our experiments. Such a cross-iteration momentum update effectively facilitates the training process. And the updated clusters would in turn have a positive effect on the node update, which achieves a local-global mutual adaptation. We summarize the entire procedure of our OEPG in Algorithm.\ref{algo-oepg}.

After the self-supervised representation learning, the derived GNN model explicitly propagates the information between the local and global contexts, discriminating the representations. Benefiting from the explicit definition of the ego-semantic descriptors, such instance-adaptive global-aware features are also applicable (in Appendix.\ref{appendix-exp}) in downstream tasks under self-supervised, semi-supervised and transfer learning settings, substantially outperforming SOTA methods.
Notably, OEPG also generalizes to the quantity- and topology-imbalance scenarios, never explored by previous self-supervised graph representation learning methods. 
\begin{algorithm}[tb]
   \caption{Algorithm of OEPG}
   \label{algo-oepg}
\begin{algorithmic}
   \STATE {\bfseries Input:} Unlabeled graph dataset $\mathbf{G}$, the number of training steps $T$. 
   \STATE {\bfseries Output:} Pre-trained GNN model $\text{GNN}_{\theta_T}$.
   \STATE Initialize the model parameters $\theta_0$ with local pre-training.
   \STATE Initialize multi-granular clusters $\{\{\mathbf{C}_{s,h}\}_{s=1}^{S_h}\}_{h=1}^{H}$.
   \STATE Initialize node queue  $\{\mathbf{Q}_{s,h}\}$ for each $\mathbf{C}_{s,h}$.
   \FOR{$t=1$ {\bfseries to} $T$}
   \STATE Sample a mini-batch subgraphs $\{\mathcal{G}_i\}_{i=1}^N \in \mathbf{G}$.
  \FOR{$i=1$ {\bfseries to} $N$}
   \IF{$size(\{\mathbf{Q_{s,h}}\})= max\  budget\  \mathcal{B}$}
   \STATE Update $\mathbf{C}_{s,h}$ with Eq.\ref{argmax},\ref{momentum}, empty $\{\mathbf{Q_{s,h}}\}$.
   \ELSE
   \STATE Enqueue target node $V_i\in \mathcal{G}_i$ to $\{\mathbf{Q_{s,h}}\}$..
   \ENDIF
   \STATE (1).\ Obtain 1st, 2nd order ego-semantic descriptors $\{\{\mathbf{D}^{1st}_{s,h}\}_{s=1}^{S_h}\}_{h=1}^{H}$, $\{\{\mathbf{D}^{2nd}_{s,h}\}_{s=1}^{S_h}\}_{h=1}^{H}$ for target node according to Eq.\ref{ego-semantic1},\ref{ego-semantic2}.
   \STATE (2).\ Omni-granular norm on $\mathbf{D}^{1st}_{s,h}$, $\mathbf{D}^{2nd}_{s,h}$, and combine them to get $\mathbf{D}_{s,h}$ with Eq.\ref{omni-4},\ref{omni-5},\ref{omni-6}.
   \STATE (3).\ Integrate $\{\{\mathbf{D}_{s,h}\}_{s=1}^{S_h}\}_{h=1}^{H}$ with $\mathcal{G}_i$ to form new subgraph $\mathcal{G}^{new}_i$ according to Eq.\ref{eq-adj}. 
  \ENDFOR
  \STATE Update model parameters $\theta_t$ according to Eq.\ref{eq-contra}, \ref{eq-predic}.
   \ENDFOR
\end{algorithmic}
\end{algorithm}

\begin{table*}[ht]
\caption{Downstream test accuracy (\%) in self-supervised learning. The compared results are from the published papers.}
\label{exp-self}
\vskip 0.1in
\begin{center}
\begin{small}
\begin{sc}
\resizebox{0.99\textwidth}{!}{
\begin{tabular}{lccccccccr}
\toprule
Methods & NCI1 &PROTEINS& DD& MUTAG& COLLAB &RDT-B& RDT-M5K &IMDB-B\\
\midrule
node2vec \citep{grover2016node2vec}
&54.9$\pm$1.6& 57.5$\pm$3.6 &- &72.6$\pm$10.2& -& - &-& -\\
sub2vec \citep{adhikari2018sub2vec}  &52.8$\pm$1.5 &53.0$\pm$5.6 &-& 61.1$\pm$15.8& -& 71.5$\pm$0.4& 36.7$\pm$0.4& 55.3$\pm$1.5 \\
graph2vec \citep{narayanan2017graph2vec} & 73.2$\pm$1.8 &73.3$\pm$2.1& -& 83.2$\pm$9.3& -& 75.8$\pm$1.0& 47.9$\pm$0.3& 71.1$\pm$0.5\\
GAE \citep{kipf2016variational} & -& -& -& 87.7$\pm$0.7 &-& 87.1$\pm$0.1& 52.8$\pm$0.2& 70.7$\pm$0.7\\
MVGRL \citep{hassani2020contrastive}  & - &-& -& 75.4$\pm$7.8 &- &82.0$\pm$1.1& -& 63.6$\pm$4.2 \\
InfoGraph \citep{sun2019infograph}   & 76.2$\pm$1.1& 74.4$\pm$0.3& 72.9$\pm$1.8& 89.0$\pm$1.1& 70.7$\pm$1.1 &82.5$\pm$1.4& 53.5$\pm$1.0& 73.0$\pm$0.9\\
GraphCL  \citep{you2020graph}   &77.9$\pm$0.4 &74.4$\pm$0.5& 78.6$\pm$0.4& 86.8$\pm$1.3& 71.4$\pm$1.2& 89.5$\pm$0.8& 56.0$\pm$0.3& 71.1$\pm$0.4\\
JOAO  \citep{you2021graph}  & 78.4$\pm$0.5 &74.1$\pm$1.1 &77.4$\pm$1.2& 87.7$\pm$0.8 &69.3$\pm$0.3 &86.4$\pm$1.5& 56.0$\pm$0.3 &70.8$\pm$0.3\\
ADGCL  \citep{suresh2021adversarial}   & 69.7$\pm$0.5& 73.8$\pm$0.5& 75.1$\pm$0.4& 89.7$\pm$1.0& 73.3$\pm$0.6 &85.5$\pm$0.8& 54.9$\pm$0.4& 72.3$\pm$0.6\\
InfoGCL   \citep{xu2021infogcl}   & 80.2$\pm$0.6& -& - & 91.2$\pm$1.3& 80.0$\pm$1.3 &-& - &75.1$\pm$0.9\\
DGCL  \citep{li2021disentangled}  & 81.9$\pm$0.2& 76.4$\pm$0.5& - & 92.1$\pm$0.2& 81.2$\pm$0.3 &92.7$\pm$0.2& 56.1$\pm$0.2& 75.9$\pm$0.7\\
\midrule
\textbf{OEPG (ours)}  & \textbf{84.8}$\pm$0.4& \textbf{79.6}$\pm$0.7& \textbf{81.4}$\pm$0.9&\textbf{95.3}$\pm$0.6&\textbf{84.7}$\pm$0.7& \textbf{96.3}$\pm$0.9& \textbf{60.5}$\pm$0.3&\textbf{78.5}$\pm$0.6 \\
\bottomrule
\end{tabular}}
\end{sc}
\end{small}
\end{center}
\vskip -0.1in
\end{table*}

\begin{table*}[ht]
\caption{Downstream test ROC-AUC (\%) in transfer learning. The compared results are from the published papers.}
\label{exp-trans}
\vskip 0.1in
\begin{center}
\begin{small}
\begin{sc}
\resizebox{0.99\textwidth}{!}{
\begin{tabular}{lccccccccr}
\toprule
Methods & BBBP & Tox21 & ToxCast & SIDER &ClinTox &MUV&HIV&BACE \\
\midrule
EdgePred \citep{kipf2016variational}  & 67.3$\pm$2.4& 76.0$\pm$0.6& 64.1$\pm$0.6& 60.4$\pm$0.7& 64.1$\pm$3.7 & 74.1$\pm$2.1& 76.3$\pm$1.0& 79.9$\pm$0.9\\
InfoGraph \citep{sun2019infograph} &68.2$\pm$0.7& 75.5$\pm$0.6 &63.1$\pm$0.3 &59.4$\pm$1.0 &70.5$\pm$1.8 &75.6$\pm$1.2 &77.6$\pm$0.4 &78.9$\pm$1.1 \\
AttrMasking \citep{hu2020gpt}  &64.3$\pm$2.8 &76.7$\pm$0.4 &64.2$\pm$0.5 &61.0$\pm$0.7 &71.8$\pm$4.1 &74.7$\pm$1.4 &77.2$\pm$1.1 &79.3$\pm$1.6\\
ContextPred \citep{rong2020self}  &68.0$\pm$2.0 &75.7$\pm$0.7 &63.9$\pm$0.6 &60.9$\pm$0.6 &65.9$\pm$3.8 &75.8$\pm$1.7 &77.3$\pm$1.0 &79.6$\pm$1.2 \\
GraphPartition \citep{you2020does} &70.3$\pm$0.7& 75.2$\pm$0.4 &63.2$\pm$0.3 &61.0$\pm$0.8 &64.2$\pm$0.5 &75.4$\pm$1.7 &77.1$\pm$0.7& 79.6$\pm$1.8\\
GraphCL \citep{you2020graph} & 69.5$\pm$0.5 &75.4$\pm$0.9 &63.8$\pm$0.4 &60.8$\pm$0.7& 70.1$\pm$1.9 &74.5$\pm$1.3& 77.6$\pm$0.9 &78.2$\pm$1.2\\
JOAO \citep{you2021graph} &71.4$\pm$0.9 &74.3$\pm$0.6& 63.2$\pm$0.5 &60.5$\pm$0.7& 81.0$\pm$1.6& 73.7$\pm$1.0& 77.5$\pm$1.2 &75.5$\pm$1.3\\
GraphLoG \citep{xu2021self} & 72.5$\pm$0.8 &75.7$\pm$0.5 &63.5$\pm$0.7 &61.2$\pm$1.1 &76.7$\pm$3.3& 76.0$\pm$1.1& 77.8$\pm$0.8 &83.5$\pm$1.2\\
\midrule
\textbf{OEPG (ours)}  & \textbf{75.7}$\pm$0.6& \textbf{79.2}$\pm$0.7& \textbf{66.2}$\pm$0.4 & \textbf{64.1}$\pm$0.9& \textbf{84.5}$\pm$1.7& \textbf{81.6}$\pm$1.4& \textbf{81.3}$\pm$0.9& \textbf{85.2}$\pm$1.3\\
\bottomrule
\end{tabular}}
\end{sc}
\end{small}
\end{center}
\vskip -0.1in
\end{table*}

\begin{table*}[ht]
\caption{Downstream test accuracy (\%) in semi-supervised learning. The compared results are from the published papers.}
\label{exp-semi}
\vskip 0.1in
\begin{center}
\begin{small}
\begin{sc}
\resizebox{0.99\textwidth}{!}{
\begin{tabular}{lcccccr}
\toprule
Methods & WikiCS & Amazon Computers & Amazon Photos & Coauthor CS &Coauthor Physics \\
\midrule
DGI \citep{velivckovic2018deep}  & 75.4$\pm$0.1 &84.0$\pm$0.5& 91.6$\pm$0.2& 92.2$\pm$0.6& 94.5$\pm$0.5 \\
GMI  \citep{peng2020graph} &74.9$\pm$0.1& 82.2$\pm$0.3& 90.7$\pm$0.2& OOM& OOM\\
MVGRL \citep{hassani2020contrastive}  & 77.5$\pm$0.1 &87.5$\pm$0.1& 91.7$\pm$0.1& 92.1$\pm$0.1& 95.3$\pm$0.1\\
GBT \citep{bielak2021graph} & 77.3$\pm$0.6& 88.0$\pm$0.3& 92.2$\pm$0.4& 92.9$\pm$0.3& 95.2$\pm$0.1 \\
GRACE \citep{zhu2020deep}  & 80.1$\pm$0.5& 89.5$\pm$0.4& 92.8$\pm$0.5& 91.1$\pm$0.2& OOM \\
GCA \citep{zhu2021graph}  &78.4$\pm$0.1& 88.9$\pm$0.2& 92.5$\pm$0.2& 93.1$\pm$0.1&95.7$\pm$0.1 \\
CCA \citep{zhang2021canonical} &-&88.7$\pm$0.3& 93.1$\pm$0.1&93.3$\pm$0.2& 95.4$\pm$0.1\\
BGRL  \citep{thakoor2021bootstrapped}  &79.4$\pm$0.5& 89.7$\pm$0.3& 92.9$\pm$0.3& 93.2$\pm$0.1& 95.6$\pm$0.1\\
\midrule
\textbf{OEPG (ours)}  & \textbf{83.3}$\pm$0.3& \textbf{91.9}$\pm$0.5& \textbf{95.1}$\pm$0.4& \textbf{95.4}$\pm$0.1& \textbf{97.3}$\pm$0.1 \\
\bottomrule
\end{tabular}}
\end{sc}
\end{small}
\end{center}
\vskip -0.1in
\end{table*}

\begin{table}[ht]
\caption{Test Macro-F1 (\%) in semi-supervised learning. The imbalance ratio (I.R.) is set to different levels (5\%, 10\%) to test under different imbalance intensities. The supervised method is \citet{chen2021topology}).}
\label{exp-imb}
\vskip 0.1in
\begin{center}
\begin{small}
\begin{sc}
\setlength{\tabcolsep}{1.0mm}{
\begin{tabular}{l|ccccccccc}
\toprule
Datasets &
\multirow{2}{*}{\space}&
\multicolumn{2}{c}{Coauthor CS} &
\multicolumn{1}{c}{\space}& 
\multicolumn{2}{c}{Coauthor Physics}
\\ \midrule 
I.R. &&5\% &10\% &&5\% & 10\%\\
\midrule
Supervised&&83.9$\pm$2.1&81.3$\pm$3.2&&72.4$\pm$2.6&70.2$\pm$2.8 \\ 
GraphCL &&75.3$\pm$4.1&72.2$\pm$5.3&&68.2$\pm$4.6& 66.5$\pm$4.9\\
JOAO&&77.8$\pm$3.9&74.6$\pm$4.5&&67.4$\pm$4.1&66.0$\pm$4.8\\
DGCL&&80.4$\pm$3.2&78.1$\pm$3.5&&69.5$\pm$3.9&67.1$\pm$3.9\\

\textbf{OEPG}&&\textbf{85.1}$\pm$1.8&\textbf{83.4}$\pm$2.1&&\textbf{75.9}$\pm$2.2&\textbf{72.6}$\pm$2.5\\

\bottomrule
\end{tabular}}
\end{sc}
\end{small}
\end{center}
\vskip -0.1in
\end{table}

\begin{table}[ht]
\caption{Semi-supervised learning on large-scale OGB datasets on  (accuracy in \% on
ogbg-ppa, F1 score in \% on ogbg-code, ROC-AUC in \% on ogbg-molhiv). L.R. denotes the label ratio.}
\label{exp-scale}
\vskip 0.1in
\begin{center}
\begin{small}
\begin{sc}
\setlength{\tabcolsep}{1.5mm}{
\begin{tabular}{clcccc}
\toprule
L.R.&Methods & ppa & code&molhiv  \\
\midrule
\multirow{4}{*}{1\%}& GraphCL &40.8$\pm$1.3&7.6$\pm$0.3&67.6$\pm$1.6 \\&JOAO&47.2$\pm$1.3&6.8$\pm$0.3 &-\\&DGCL&-&-&69.0$\pm$1.7 \\&\textbf{OEPG}&\textbf{52.7}$\pm$1.2&\textbf{9.1}$\pm$0.4&\textbf{74.2}$\pm$1.6\\

\midrule
\multirow{4}{*}{10\%}& GraphCL &57.8$\pm$1.3&22.5$\pm$0.2&70.6$\pm$1.6  \\&JOAO&60.9$\pm$0.8&22.1$\pm$0.3&-\\&DGCL &-&-&73.6$\pm$1.5\\&\textbf{OEPG}&\textbf{64.4}$\pm$0.7&\textbf{24.8}$\pm$0.3&\textbf{77.5}$\pm$1.7\\
\bottomrule
\end{tabular}}
\end{sc}
\end{small}
\end{center}
\vskip -0.1in
\end{table}
\section{Experiments and Analysis}
\label{sec-exp-analysis}
\subsection{Experiment Setups}
\paragraph{Datasets of Different Scales and Domains.} To adequately validate the effectiveness of our OEPG, we use the mutiple downstream datasets cross the scales (small, medium and large scales) and domains (social, academic and biomedical graphs) including TUDataset \cite{morris2020tudataset}, WikiCS \cite{mernyei2020wiki}, Amazon Computers\&Amazon Photos \cite{mcauley2015image}, Coauthor CS\&Coauthor Physics \cite{sinha2015overview}, MoleculeNet \cite{wu2018moleculenet}, Citeseer, Cora, Pubmed \cite{sen2008collective}, Open Graph Benchmark (OGB) \cite{hu2020open}. More descriptions and statistics are in Appendix.\ref{appendix-data}.
\paragraph{Evaluation Protocols.} To adequately evaluate our algorithm, we perform downstream evaluations in three settings, following the same protocols as in State-of-the-Art. (1) In self-supervised learning \cite{xu2021self}, we pre-train the model using the whole graph dataset to learn a graph encoder and compute embedding for each node or graph, then the embeddings are feed into a downstream SVM \cite{you2021graph} or linear classifier \cite{velivckovic2018deep} for evaluation. (2) In transfer learning, we first pre-train the model on a larger graph dataset \cite{you2021graph,xu2021self}. Then we finetune and evaluate the model on smaller datasets of the same category using the given training/validation/test split. (3) In semi-supervised learning \cite{you2020graph}, for datasets without explicit train/validation/test split, we conduct pre-training process with all graph data at first. Then we finetune and evaluate the model with K folds \cite{you2021graph}. For datasets with the explicit split, we only pre-train the model with the training split, finetune on the partial training split and evaluate on the validation/test splits.
\paragraph{GNN Architectures.}We use the same
GNN architectures with default training hyper-parameters as in the SOTA methods under three experiment settings. In specific, (1)
in self-supervised representation learning, GIN \cite{xu2018powerful} is used with 3 layers and 32 hidden dimensions, (2) in semi-supervised learning, a 2-layer or 3-layer (for large graphs) GCN encoder network with 128 or 256 hidden dimensions, and
(3) in transfer learning and on large-scale OGB datasets,
GIN is used with 5 layers and 300 hidden dimensions.

\subsection{Experimental Comparisons with State-of-the-Art}
We conduct a thorough comparison with previous unsupervised/self-supervised methods, including: (1) traditional graph embedding such as node2vec \cite{grover2016node2vec}, sub2vec \cite{adhikari2018sub2vec},
graph2vec \cite{narayanan2017graph2vec} and GAE \cite{kipf2016variational}, (2) contrastive learning including MVGRL \cite{hassani2020contrastive}, GraphCL \cite{you2020graph}, JOAO \cite{you2021graph}, DGCL   \cite{li2021disentangled}, GCA \cite{zhu2021graph} and GRACE \cite{zhu2020deep}.
(3) information theory including InfoGraph \cite{sun2019infograph}, ADGCL \cite{suresh2021adversarial}, InfoGCL  \cite{xu2021infogcl}, DGI \cite{velivckovic2018deep}, CCA \cite{zhang2021canonical} and GMI \cite{peng2020graph},
(4) predictive training paradigm including BGRL  \cite{thakoor2021bootstrapped}, EdgePred  \cite{kipf2016variational}, AttrMasking \cite{hu2020gpt} and ContextPred \cite{rong2020self},
(5) others such as GBT \cite{bielak2021graph}, GraphPartition \cite{you2020does} and GraphLoG  \cite{xu2021self}. \textbf{See Appendix.\ref{appendix-exp} for more details.} 
\paragraph{Self-Supervised Learning.} We follow the previous methods \cite{you2021graph} to perform pre-training and evaluation, and report the results in Table.\ref{exp-self}. Our OEPG substantially outperforms all previous methods with \textbf{a 3\%$\sim$4\% accuracy gain} compared to SOTA. It demonstrates that OEPG exploits the local proximity and global diversity, improving the expressiveness and discriminativity of repersentations.
\paragraph{Transfer Learning.} For fair comparison, we use the same subset of ZINC15 \cite{sterling2015zinc} for pre-training as \cite{xu2021self}, and perform evaluation on smaller datasets with the results in Table.\ref{exp-trans}. Among all methods, our OEPG achieves the best performance on all datasets, and \textbf{a 2\%$\sim$6\% performance gain} is obtained in terms of ROC-AUC compared to SOTA method.
\paragraph{Semi-Supervised Learning.} We report our comparison results in Table.\ref{exp-semi} with the same evaluation setting as \cite{bielak2021graph}. Our OEPG surpasses the SOTA by a large margin, with \textbf{a 2\%$\sim$4\% accuracy improvement}. The proposed model could strengthen the label propagation with the effective local-global semantic feature propagation.
\paragraph{Generalizing to Quantity- and Topology-Imbalance Scenarios.}
Quantity-imbalance denotes the unequal quantity of labeled examples while topology-imbalance \cite{chen2021topology} means the asymmetric topological properties of the labeled nodes. They are both practical issues and hard to address. Notably, we find that our OEPG also generalizes well in tasks where quantity- and topology-imbalance problems co-occur as reported in Table.\ref{exp-imb}, outperforming previous methods \textbf{by 2\%$\sim$7\% including supervised model}. The main reason is OEPG enables each node to adequately embed the global semantic features to enrich and discriminate itself despite sparse samples as illustrated in Appendix.\ref{app-prove}.
\paragraph{Scaling to Large-Scale Graph Datasets.} We evaluate our OEPG on large-scale OGB datasets, with results in Table.\ref{exp-scale}. Note that our OEPG significantly outperforms the previous SOTA methods with \textbf{2\%$\sim$5\% performance gains} at 1\% and 10\% label
ratios, demonstrates the scalability of OEPG.


\subsection{Model Analysis}
\paragraph{Effectiveness of Ego-Semantic Descriptors.} We investigate the effect of our proposed ego-semantic descriptors on downstream tasks and we reproduce GraphCL as our baseline. We report the result on the MUV dataset in Figure.\ref{pic-ablation}. The performance is obviously promoted by 2\%-4\% when adding the ego-semantic descriptors, which demonstrates the sufficient utilization of global information in OEPG. Besides, we also conclude that second-order ego-semantic is more effective than first-order one, and the concatenation of them performs the best with the diverse global perspective.  

\paragraph{The Impact of Omni-Granular Normalization.} In Subsection.\ref{sec-omni}, we have theoretically proved the value of parameters $\alpha,\beta$ can reflect the global diversity of graph dataset. We also provide the experimental proof in Figure.\ref{pic-ablation-omni} of Appendix.\ref{appendix-analysis}. We gradually increase the diversity of graph datasets, and find the values of $\alpha,\beta$ become correspondingly larger. The phenomenon implies the omni-granular normalization effectively reflects and leverages the global diversity. We also find the second-order ego-descriptors (related with $\beta$) are more sensitive to the global diversity.
\begin{figure}[ht]
\vskip 0.1in
\begin{center}
\centerline{\includegraphics[width=8.5cm]{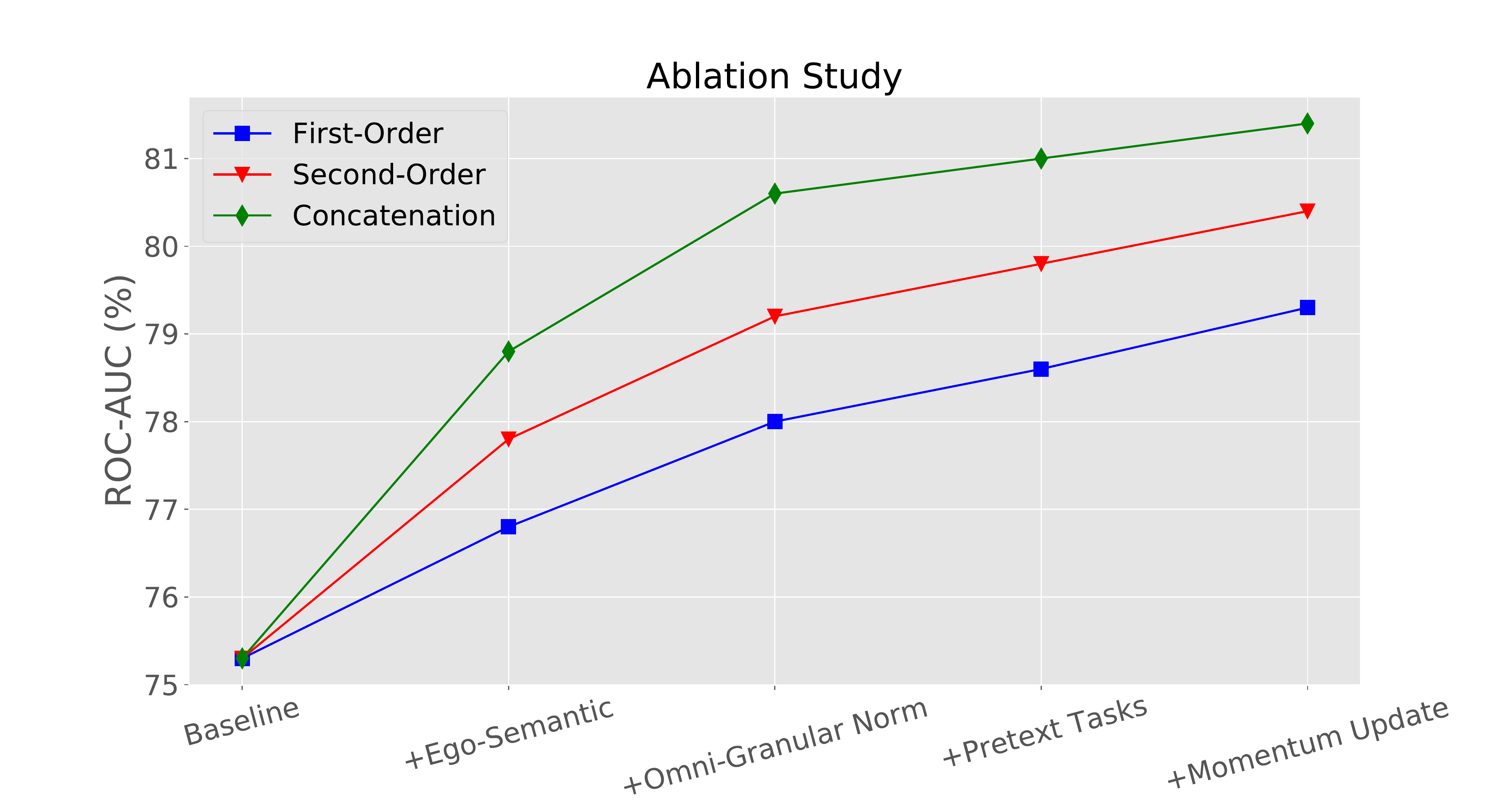}}
\caption{Ablation study on MUV dataset.}
\label{pic-ablation}
\end{center}
\vskip -0.2in
\end{figure}
\paragraph{Ablation Study of OEPG.} Illustrated in Section.\ref{sec-oepg}, OEPG consists of four cascade modules. The ablation study is conducted by gradually adding each module in order, and the results are in Figure.\ref{pic-ablation}.
We conclude that each module gives OEPG a accuracy gain, demonstrates the efficacy of the designs in our framework. The x-axis in is order invariant since the modules are conducted in sequence. Note that the ego-semantic descriptors provide the most performance gain among all modules, further enhanced when empowered by omni-granular normalization. And specified pretext tasks and cross-iteration momentum update are also vital for the integrity and effectiveness of our framework. The results on more datasets are illustrated in Appendix.\ref{response-1.2}. More analysis about the sensitivity of the hyperparameters are in Appendix.\ref{appendix-analysis}. 

\section{Conclusions}
In this paper, we propose a unified framework to explicitly characterize the global-aware feature and enable the local-global semantic propagation, namely Omni-Granular Ego-Semantic Propagation for self-supervised Graph representation learning (OEPG). We first propose the fused first- and second-order ego-semantic descriptors, weighted by omni-granular normalization, to explicitly represent the instance-adaptive global semantic information. New local-global pretext tasks and a cross-iteration omni-granular momentum update are further developed to improve the local-global mutual adaptation. Extensive experiments on datasets cross scales and domains are conducted and our OEPG substantially outperforms all the SOTA methods, demonstrating the effectiveness. Notably, OEPG also achieves the best performance in the quantity- and topology-imbalance scenarios. For the future work, we will further improving the scalability and transferability of our framework.

\section*{Acknowledgement}
This work was supported by the National Natural Science Foundation of China (No.62102008).

\bibliography{example_paper}
\bibliographystyle{icml2022}

\newpage
\appendix
\onecolumn
\section{Dataset Descriptions and Statistics}
\label{appendix-data}
In this section, we provide the detailed descriptions and statistics about all the datasets used in Section.\ref{sec-exp-analysis}.

\paragraph{TUDataset} \cite{morris2020tudataset} contains the datasets of different natures with graph data for small molecules \& proteins, computer vision and various relation networks.
\begin{table*}[ht]
\caption{Statistics for datasets of diverse nature from the benchmark TUDataset.}
\vskip 0.15in
\begin{center}
\begin{small}
\begin{sc}
\setlength{\tabcolsep}{1.5mm}{
\begin{tabular}{c|c|c|cc}
\toprule
Dataset& Graph Count& Avg. Node& Avg. Degree\\\midrule
NCI1& 4,110 &29.87& 1.08\\
PROTEINS& 1,113& 39.06& 1.86\\
DD& 1,178& 284.32 &715.66\\
MUTAG &188 &17.93& 19.79\\
COLLAB &5,000 &74.49& 32.99\\
RDT-B& 2,000& 429.63 &1.15\\
RDB-M &2,000 &429.63 &497.75\\
GITHUB& 4,999& 508.52& 594.87\\
IMDB-B& 1,000 &19.77 &96.53\\
\bottomrule
\end{tabular}}
\end{sc}
\end{small}
\end{center}
\vskip -0.1in
\end{table*}
\paragraph{WikiCS}  \cite{mernyei2020wiki} is a network of Computer Science related Wikipedia articles with edges
representing references between those articles. \textbf{Amazon Computers, Amazon Photos} \cite{mcauley2015image} are two networks extracted from Amazon’s copurchase data. Nodes are products and edges denote that these products were often bought together. \textbf{Coauthor CS, Coauthor Physics} are two networks extracted from the Microsoft Academic
Graph \cite{sinha2015overview}. Nodes are authors and edges denote a collaboration of two authors.
\begin{table*}[ht]
\caption{Statistics for datasets of WikiCS, Amazon Computers, Amazon Photos, Coauthor CS and Coauthor Physics.}
\vskip 0.15in
\begin{center}
\begin{small}
\begin{sc}
\setlength{\tabcolsep}{1.5mm}{
\begin{tabular}{c|c|c|c|c}
\toprule
Dataset& Nodes&  Edges&  Features & Classes\\\midrule
WikiCS&  11,701 & 216,123 & 300 & 10\\
Amazon  Computers & 13,752&  245,861 & 767&  10\\
Amazon Photos&  7,650&  119,081&  745&  8\\
Coauthor CS&  18,333&  81,894 & 6,805 & 15\\
Coauthor Physics & 34,493&  247,962&  8,415 & 5\\
\bottomrule
\end{tabular}}
\end{sc}
\end{small}
\end{center}
\vskip -0.1in
\end{table*}
\paragraph{Citeseer, Cora, Pubmed} \cite{sen2008collective} are widely-used citation networks where documents (nodes)
are connected through citations (edges).
\begin{table*}[ht]
\caption{Statistics for datasets of Citeseer, Cora and Pubmed.}
\vskip 0.15in
\begin{center}
\begin{small}
\begin{sc}
\setlength{\tabcolsep}{1.5mm}{
\begin{tabular}{c|c|c|c|c}
\toprule
Dataset  &Nodes&Edges &Features &Classes\\\midrule
Cora & 2,708 &5,429& 1,433& 7\\
Citeseer &3,327 &4,732 &3,703 &6\\
Pubmed &19,717 &44,338 &500& 3\\
\bottomrule
\end{tabular}}
\end{sc}
\end{small}
\end{center}
\vskip -0.1in
\end{table*}
\paragraph{MoleculeNet} \cite{wu2018moleculenet} has a full
collection currently includes over 700 000 compounds tested on a range of different properties. These properties can be subdivided into four categories: quantum mechanics, physical
chemistry, biophysics and physiology.
\begin{table*}[ht]
\caption{Statistics for datasets of diverse nature from the benchmark MoleculeNet.}
\vskip 0.15in
\begin{center}
\begin{small}
\begin{sc}
\setlength{\tabcolsep}{1.5mm}{
\begin{tabular}{c|c|c|cc}
\toprule
Dataset &Graph Count& Avg. Node& Avg. Degree\\\midrule
BBBP &2,039 &24.06 &51.90\\
Tox21& 7,831& 18.57& 38.58\\
ToxCast& 8,576 &18.78 &38.52\\
SIDER &1,427& 33.64& 70.71\\
ClinTox& 1,477 &26.15 &55.76\\
MUV &93,087& 24.23 &52.55\\
HIV &41,127 &25.51& 54.93\\
BACE &1,513 &34.08 &73.71\\

\bottomrule
\end{tabular}}
\end{sc}
\end{small}
\end{center}
\vskip -0.1in
\end{table*}
\paragraph{Open Graph Benchmark (OGB)} \cite{hu2020open} contains a diverse set of large-scale
and realistic benchmark datasets to validate the scalablity and robustness of graph machine learning algorithms.
\begin{table*}[ht]
\caption{Statistics for large-scale datasets of diverse nature from the OGB benchmark.}
\vskip 0.1in
\begin{center}
\begin{small}
\begin{sc}
\setlength{\tabcolsep}{1.5mm}{
\begin{tabular}{c|c|c|c}
\toprule
Dataset& Graph Count& Avg. Node& Avg. Degree\\\midrule
ogbg-ppa &158,110 &243.4 &2,266.1\\
ogbg-code& 452,741& 125.2& 124.2\\
ogbg-molhiv&41127&25.5&27.5\\
\bottomrule
\end{tabular}}
\end{sc}
\end{small}
\end{center}
\vskip -0.1in
\end{table*}
\section{More Experimental Details, Results and Analysis}
\label{appendix-exp}
\subsection{More Experimental Details}
\paragraph{Pre-Training Details.}
In the pre-training process, we use an Adam optimizer \cite{kingma2014adam} (learning rate: $1\times 10^{-3}$
) to pre-train the OEPG model with the typical local graph contrastive learning \cite{you2020graph} for several epochs and then train the whole model with the specified pretext tasks, optimized by the contrastive loss $\mathcal{L}_{contrastive}$  in Eq.\ref{eq-contra} and predictive loss $\mathcal{L}_{predictive}$ in Eq.\ref{eq-predic}. The total epochs are same or comparable with previous methods.
At the start of every two epochs, we would perform K-means clustering for the initialization of hierarchical clusters. The parameter budget $\mathcal{B}$ is set to 4 and we adopt the 4-hierarchies ($H=4$) ego-semantic descriptors for each subgraph/graph, with the specified cluster numbers $[S_1,S_2,S_3,S_4]=[16, 12, 8, 4]$. 
\paragraph{Fine-tuning Details.} For fine-tuning on a certain downstream task, we follow the default setting in the previous methods \cite{xu2021self,you2021graph,bielak2021graph}. There are some differences between self-supervised/semi-supervised learning and transfer learning when applying our OEPG. In self-supervised/semi-supervised learning tasks, the classes between the pre-training and fine-tuning are the same, so the hierarchical clusters can be straightly used in computing the ego-semantic descriptors for downstream classification tasks. In contrary,  the hierarchical clusters will be recomputed for obtaining the ego-semantic descriptors in transfer learning downstream tasks since the classes have changed from the pre-training to the fine-tuning.
\subsection{More Experimental Results}
\begin{table*}[ht]
\caption{Downstream test accuracy (\%) on classification tasks.}
\label{app-res}
\vskip 0.15in
\begin{center}
\begin{small}
\begin{sc}
\begin{tabular}{lcccccr}
\toprule
Methods & CORA & CITESEER & PUBMED  \\
\midrule
DGI    & 83.8$\pm$0.5& 72.0$\pm$0.6& 77.9$\pm$0.3\\
MVGRL  & 86.8$\pm$0.5& 73.3$\pm$0.5& 80.1$\pm$0.7\\
GRACE    & 83.0$\pm$0.8& 71.6$\pm$0.6& 86.0$\pm$0.2\\
CCA& 84.2$\pm$0.4&73.1$\pm$0.3&81.6$\pm$0.4\\
BGRL   &83.8$\pm$1.6& 72.3$\pm$0.9& 86.0$\pm$0.3\\
\midrule
\textbf{OEPG (ours)}  & \textbf{88.3}$\pm$0.5& \textbf{75.5}$\pm$0.8& \textbf{88.7}$\pm$0.4\\
\bottomrule
\end{tabular}
\end{sc}
\end{small}
\end{center}
\vskip -0.1in
\end{table*}
\paragraph{More Comparison Results.}
Due to the page limit, we have one experiment comparison results left as in Table.\ref{app-res}. As in the main text, our OEPG also achieves the best performance in all methods with a large accuracy promotion.

\subsection{More Analysis of OEPG}
\label{appendix-analysis}

\paragraph{The Discriminativity of Ego-Semantic Descriptors.}
\label{app-prove}
In this paper, we mainly propose the ego-semantic descriptors in the graph feature propagation process to discriminating the final representations. Next, we will illustrate how it works in specific and we take graph-level pre-training for example. Given the full graph dataset: 
\begin{align}
\label{ego-prove}
    \{\underbrace{\mathcal{G}_1, \mathcal{G}_2,\cdots, \mathcal{G}_g}_{\mathbf{C}_1}, \underbrace{\mathcal{G}_{g+1}, \cdots, \mathcal{G}_{2\cdot g}}_{\mathbf{C}_2},\cdots,\underbrace{\mathcal{G}_{(n-1)\cdot g+1} \cdots, \mathcal{G}_{n\cdot g}}_{\mathbf{C}_n}\},
\end{align}
where we assume each graph cluster $\mathbf{C}_n$ contains a group of same number of graphs and only consider 1-hierarchy clusters for simplicity. When we judge whether a pair of graphs $\mathcal{G}_i, \mathcal{G}_j$ belong to the same class or not, we typically compute the similarity between the graph embeddings $f_{\mathcal{G}_i}, f_{\mathcal{G}_j}$ as follow:
\begin{align}
    sim(f_{\mathcal{G}_i},\ f_{\mathcal{G}_j})=sim(\mathcal{F}(\mathcal{G}_i),\  \mathcal{F}(\mathcal{G}_j)), 
\end{align}
where $\mathcal{F}$ denotes all the linear and non-linear operations and $sim(\cdot,\cdot)$ is the similarity function. Equipped with ego-semantic descriptors, the process can be reformulated as follow:
\begin{align}
\begin{split}
    sim(f_{\mathcal{G}_i},\ f_{\mathcal{G}_j})=sim(&\mathcal{F}(\mathcal{G}_i,\underbrace{\{\mathcal{F}(\mathcal{G}_i)-\mathbf{C}_k\}^n_{k=1}}_{1st-order\ ego-semantic}),\quad \mathcal{F}(\mathcal{G}_j,\underbrace{\{\mathcal{F}(\mathcal{G}_j)-\mathbf{C}_k\}^n_{k=1}}_{1st-order\ ego-semantic})), \\
    sim(f_{\mathcal{G}_i},\ f_{\mathcal{G}_j})=sim(&\mathcal{F}(\mathcal{G}_i,\underbrace{\{(\mathcal{F}(\mathcal{G}_i)-\mathbf{C}_l)\bullet \{(\mathcal{F}(\mathcal{G}_i)-\mathbf{C}_k)\}_{k=1}^{n}\}^n_{l=1}}_{2nd-order\ ego-semantic}),\quad\\ &\mathcal{F}(\mathcal{G}_j,\underbrace{\{(\mathcal{F}(\mathcal{G}_j)-\mathbf{C}_l)\bullet \{(\mathcal{F}(\mathcal{G}_j)-\mathbf{C}_k)\}_{k=1}^{n}\}^n_{l=1}}_{2nd-order\ ego-semantic})), 
\end{split}
\end{align}
since clusters $\{\mathbf{C}_k\}_{k=1}^n$ are all non-parametric, and thus ego-semantic descriptors only vary with $\mathcal{G}_i, \mathcal{G}_j$, which are graph-adaptive. When $\mathcal{G}_i, \mathcal{G}_j$ are of different classes, all their ego-semantic descriptors can be different, enlarging the distances between $f_{\mathcal{G}_i}$ and $f_{\mathcal{G}_j}$.  Conversely, if $\mathcal{G}_i, \mathcal{G}_j$ are of the same class, all their ego-semantic descriptors can be similar, minimizing the distances between $f_{\mathcal{G}_i}$ and $f_{\mathcal{G}_j}$. The proposed ego-semantic descriptors not only discriminate the representations but also enhance the expressiveness by explicitly embedding the instance-adaptive global semantic features.
\begin{figure}[ht]
\vskip 0.2in
\begin{center}
\centerline{\includegraphics[width=9cm]{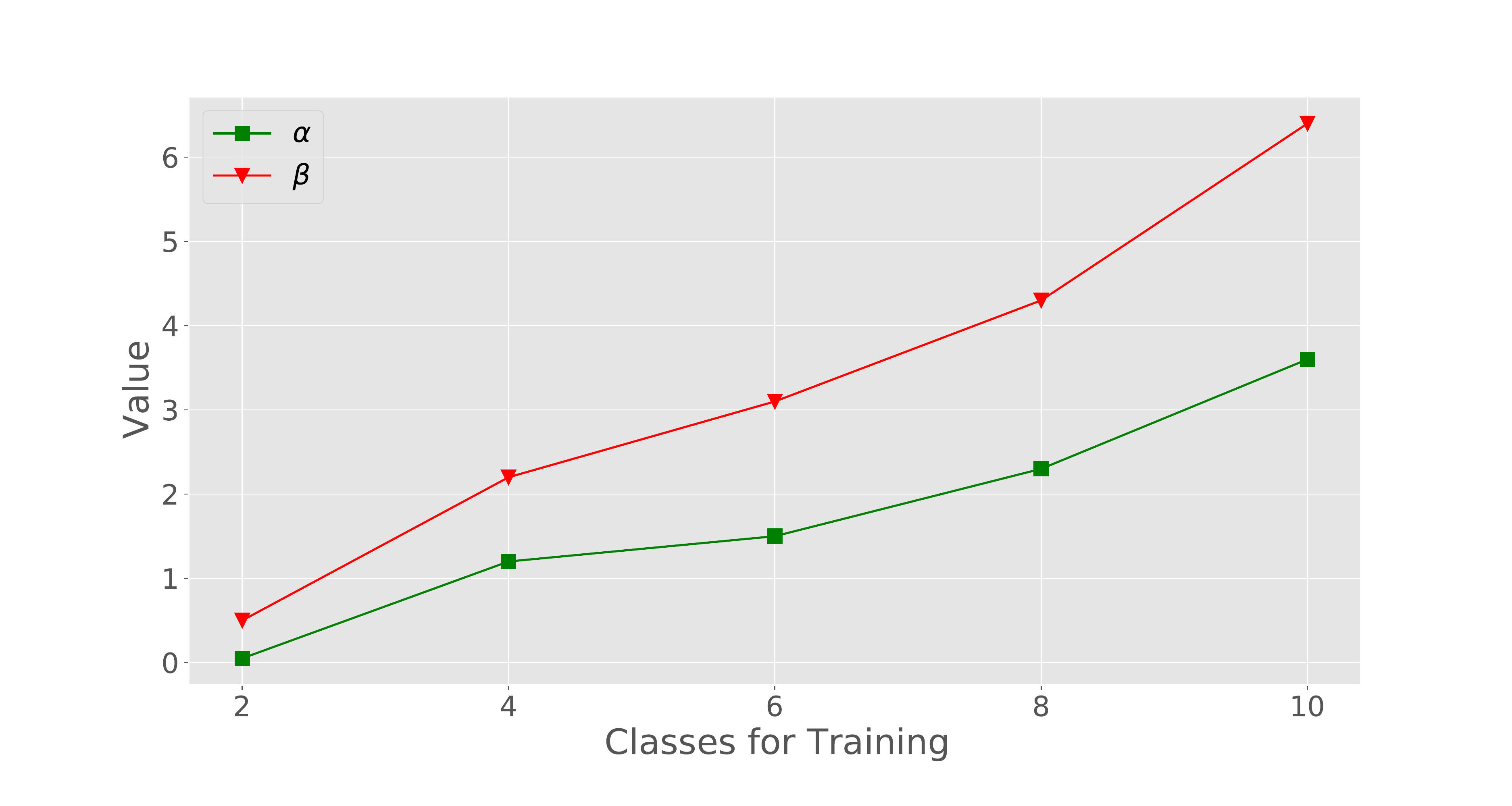}}
\caption{Experimental analysis involving omni-granular normalization.}
\label{pic-ablation-omni}
\end{center}
\vskip -0.2in
\end{figure}
\paragraph{Experimental Analysis Involving Omni-Granular Normalization.} We make the controllable experiments to validate the effectiveness of omni-granular normalization. We use the 10-class WikiCS dataset for illustration. Specifically, we pre-train the OEPG on the graph dataset with the different number of total classes (from 2 to 10 classes). And we print the learned value of $\alpha,\beta$ and report them in Figure.\ref{pic-ablation-omni}. We can find that with the increasing classes (global diversity), the values also become larger and the value of $\beta$ (related with second-order ego-semantic descriptors) is more sensitive to the classes increasing. 

\begin{table}[ht]
\caption{Ablation study on more datasets (accuracy in \%.).}
\vskip 0.03in
\label{exp-ablation}
\centering
\setlength{\tabcolsep}{1.7mm}{
\begin{tabular}{l|cccccc}
\toprule
Modules& RDT-B& WikiCS \\
\midrule
 Baseline& 91.6&78.1\\
 +Ego-Semantic&93.8 (+2.2)& 80.4 (+2.3)\\
 +Omni-Granular Norm&95.0 (+1.2)&82.0 (+1.6)\\

 +Pretext Tasks &95.7 (+0.7)& 82.8 (+0.8)\\
 +Momentum Update &96.3 (+0.6)&83.3 (+0.5)\\
\bottomrule
\end{tabular}}
\end{table}
\paragraph{More Ablation Results.}
\label{response-1.2}
Here we provide more ablation results in Table.\ref{exp-ablation}. 
The results including Fig.\ref{pic-ablation} consistently illustrate that the \textbf{Ego-Semantic is most relevant contribution to accuracy gains}, because it is the Ego-Semantic which explicitly introduces the additional global information to the representations and other modules are based on it to make optimizations. 
\begin{figure}[ht]
\vskip 0.2in
\begin{center}
\centerline{\includegraphics[width=10cm]{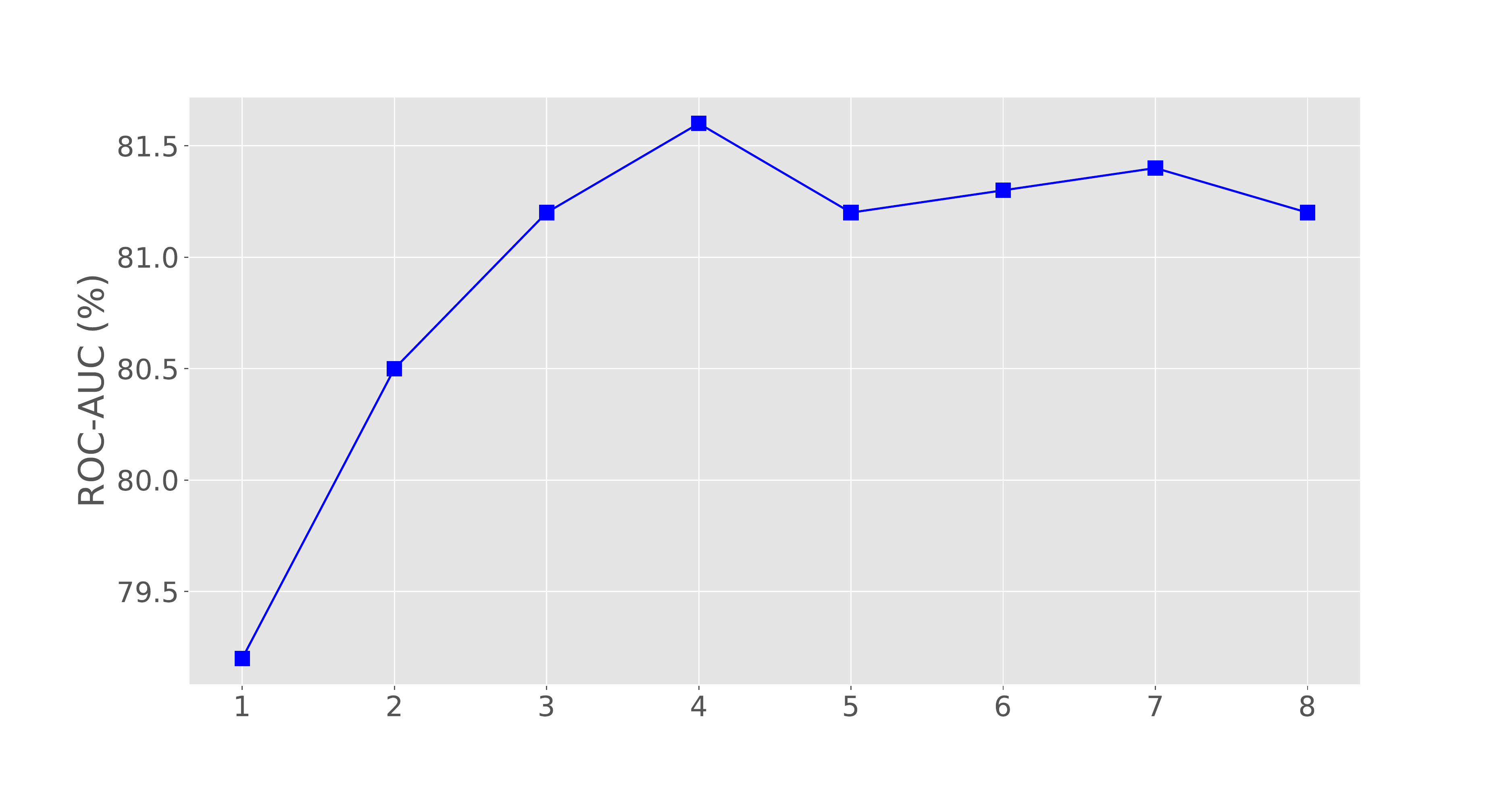}}
\caption{Ablation study of the hierarchies of the ego-semantic descriptors.}
\label{pic-ablation-hyp}
\end{center}
\vskip -0.2in
\end{figure}
\paragraph{Sensitivity of Hierarchy $h$ in Ego-Semantic.} In this part, we discuss the selection of parameter $h$ in our OEPG and report the results on MUV dataset for example. From the Figure.\ref{pic-ablation-hyp}, we can see that the performance is gradually promoted with the higher $h$, which implies that our omni-granular ego-semantic does enable the model to capture the global diversity and discriminate the representations. And the performance has a slight degradation when $h$ is higher than 4. We think it may be caused by the large redundancy of the ego-semantics.


\end{document}